%% file: main_bayesbound.tex
\documentclass{article}
\usepackage{authblk}
\usepackage[utf8]{inputenc}

\title{Rate-optimal Bayesian Simple Regret in Best Arm Identification}

\author[1]{Junpei Komiyama\thanks{Corresponding Author: junpei@komiyama.info}}
\author[2,3]{Kaito Ariu}
\author[2]{Masahiro Kato}
\author[4]{Chao Qin}
\affil[1]{Stern School of Business, New York University}
\affil[2]{AI Lab, CyberAgent, Inc.}
\affil[3]{School of Electrical Engineering and Computer Science, KTH}
\affil[4]{Columbia Business School, Columbia University}

\date{\today}

\usepackage{natbib}
\usepackage{graphicx}
\usepackage[driver=dvipdfm,margin=1in]{geometry}
\usepackage{setspace}
\usepackage{amsmath,amssymb,amsthm}
\usepackage[at]{easylist}

\input{preamble.tex}

\begin{document}

\maketitle

\onehalfspacing

\begin{abstract}
We consider best arm identification in the multi-armed bandit problem.
Assuming certain continuity conditions of the prior, we characterize the rate of the Bayesian simple regret.
Differing from Bayesian regret minimization (Lai, 1987), the leading term in the Bayesian simple regret derives from the region where the gap between optimal and suboptimal arms is smaller than $\sqrt{\frac{\log T}{T}}$.
We propose a simple and easy-to-compute algorithm with its leading term matching with the lower bound up to a constant factor; simulation results support our theoretical findings.
\end{abstract}

\section{Introduction}

We consider finding the best treatment among $K$ and $T$ sample size.
In this problem, each arm (treatment) $i \in [K] = \{1,2,\dots,K\}$ is associated with (unknown) parameter $\mu_i \in [0,1]$.
We use $\bmu = (\mu_1,\mu_2,\dots,\mu_K)$ to denote the set of parameters. 
At each round $t =1,2,\dots,T$, the forecaster, who follows some adaptive algorithm, selects an arm $I(t) \in [K]$ and receives the corresponding reward $X_{I(t)}(t) \sim \Bernoulli(\mu_{I(t)})$, where $1$ and $0$ represent the success and the failure of the selected treatment. Let $\ist = \argmax_i \mu_i$ 
and $\mu^* = \mu_{\ist}$ be the optimal (best) arm and its corresponding mean, respectively.\footnote{Ties are broken arbitrarily.} 

The existing literature has mainly considered two different objectives. The first involves maximizing the total reward \citep{Robbins1952,Lairobbins1985}, which is equivalent to minimizing the draw of suboptimal arms. Letting $N_i(T)$ be the number of draws on arm $i$ up to round $T$, the (frequentist) regret is defined as 
\begin{equation}\label{ineq_freqreg}
\Regret_{\bmu}(T) := \sum_{i=1}^K \Ex_{\bmu}[N_i(T)] (\mu^* - \mu_i),
\end{equation}
where $\bmu$ is unknown but fixed, and 
the expectation here is over the randomness of the rewards and (possibly randomized) choices $I(t)$.

The second objective is identifying the best arm. In this case, the forecaster at the end of round $T$ recommends an arm $J(T)$. The performance of the forecaster is measured by the simple regret
\begin{equation}\label{ineq_freqsreg}
\SRegFreq(T) := \mu^* - \Ex_{\bmu}[\mu_{J(T)}],
\end{equation}
which is the expected difference between the means of the best arm and recommended arm $J(T)$. The two objectives are very different. 
Minimizing regret (Eq.~\eqref{ineq_freqreg}) requires balancing the exploration (i.e., drawing all arms to obtain more information) and the exploitation (i.e., obtaining more rewards by drawing an empirically good arm). On the other hand, in minimizing simple regret, the rewards from the arms $I(t)$ ($1\le t\le T$) are not considered.
Therefore, Eq.~\eqref{ineq_freqsreg} is minimized by \textit{pure exploration} \citep{Bubeck2011}. 

The simple regret of Eq.~\eqref{ineq_freqsreg} is frequentist; it assumes that $\bmu$ is (unknown but) fixed. On the other hand, we may consider a distribution of $\bmu \in [0,1]^K $ and take the expectation of the frequentist simple regret over the distribution. We can call this the Bayesian simple regret,\footnote{In decision theory, this quantity is referred to as the Bayes risk.} which is defined as
\begin{align}\label{ineq_reg_bayes}
\SRegBayes(T) 
&= \Ex_{\bmu \sim H}[ \SRegFreq(T) ],
\end{align}
where $\Ex_{\bmu \sim H}$ marginalizes $\bmu$ over the prior $H$ on $[0,1]^K$.
In this paper, we consider the problem of minimizing the Bayesian simple regret of 
Eq.~\eqref{ineq_reg_bayes}. We drop the term ``Bayesian'' when it clearly refers to Bayesian simple regret.

\subsection{Regularity condition}

We assume the following regularity condition for the prior distribution. 
For $i \in [K]$, let $\bmuex{i}$ be the set of $K-1$ parameters other than $\mu_i$.
For $i,j \in [K]$, let $\bmuex{ij}$ be the set of $K-2$ parameters other than $\mu_i,\mu_j$.
Let $\Hex{i}(\bmuex{i})$ be the joint cumulative density function of $\bmuex{i}$, and $\Hcond{i}(\mu_i|\bmuex{i})$ be the conditional cumulative density function of $\mu_i$ given $\bmuex{i}$. 
Define $\Hex{ij}(\bmuex{ij})$, $\Hcond{ij}(\mu_i,\mu_j|\bmuex{ij})$ in the same way. The following assumption concerns the existence of continuous derivatives of $\Hcond{i}(\mu_i|\bmuex{i})$ and $\Hcond{ij}(\mu_i,\mu_j|\bmuex{ij})$.

\begin{assp}{\rm (Uniform continuity of the conditional probability density functions)}\label{assp_cont}
There exist conditional probability density functions $h_i(\mu_i|\bmuex{i})$ and $h_{ij}(\mu_i,\mu_j|\bmuex{ij})$ that are uniformly continuous.
Namely, for every $\eps>0$ there exists $\delta = \delta(\eps)>0$ such that
\begin{alignat}{2}
&\forall\,|\mu_i-\lambda_i|\le\delta,  & &\quad|h_i(\mu_i|\bmuex{i}) - h_i(\lambda_i|\bmuex{i})|
\le \eps, \nn 
&\forall\,|\mu_i-\lambda_i|,|\mu_j-\lambda_j|\le\delta, & &\quad|h_{ij}(\mu_i,\mu_j|\bmuex{ij}) - h_{ij}(\lambda_i,\lambda_j|\bmuex{ij})|
\le \eps. \label{ineq_uniformcont}
\end{alignat}
\end{assp}
\begin{remark}{\rm (Uniform continuity)}
Assumption \ref{assp_cont} is similar to that of \citet{lai1987}\footnote{Eq.~(3.17) in Theorem~3 of \citet{lai1987}.}; however, it is slightly stronger.
Namely, 
we assume the uniform continuity; $\delta$ in Eq.~\eqref{ineq_uniformcont} does not depend on $\bmu$. 
We also assume the uniform continuity of $h_{ij}(\mu_i,\mu_j|\bmuex{ij})$, which is required to bound the probability that three or more arms have very similar means.
\end{remark}
We consider Assumption \ref{assp_cont} to be satisfied by most  distributions of interests. For example, it is satisfied when the joint distribution is Lipschitz continuous, as in the case that each $\mu_i$ is drawn from the uniform prior and not very strongly correlated with each other. However, the following demonstrates a situation in which Assumption \ref{assp_cont} does not hold.

\begin{example}{\rm (Corner case excluded by Assumption \ref{assp_cont})}
Let there be three arms. Parameters $\mu_1, \mu_2 \sim \Unif(0,1)$ are independent each other, and $\mu_3 = 1 - \mu_1$. That is, the parameters are redundant. This case violates Assumption \ref{assp_cont} because $h_3(\mu_3|\mu_1,\mu_2)$ has point mass on $1-\mu_1$ and is not continuous.
\end{example}

Note also that we assume that the prior $H$ is independent of $T$. In appendix, we describe another example in which $H$ is dependent on $T$ (Section \ref{sec_supplexpl} in the appendix).

\subsection{Main results}

Table \ref{tbl_comp} compares our results and existing results.\footnote{Note that ``up to a constant'' in our results for the Bayesian SRM is stronger than the ``up to a constant'' in the frequentist SRM. 
On one hand, our bound is uniform; the constant in our bound does not depend on $\bmu$ nor $K$.
On the other hand, bound for the frequentist SRM is optimal up to a constant just for one instance $\bmu$:
\citet{Carpentier2016} showed the lower bound of $\exp(- C \Gamma T)$ for some constant $C>0$, where $\Gamma = (\log(K) \sum_{i \ne \ist}(\mu^*-\mu_i)^{-2})^{-1}$. While this bound is optimal for just one instance, the constant may diverge or the $\log(K)$ term may be unnecessary for some other instances of $\bmu$.
}
Subsequently, we characterize the optimal rate of the Bayesian simple regret.
\begin{table}[h!]
\begin{center}
\caption{Optimal rate in regret minimization (RM) setting and simple regret minimization (SRM) setting. The optimality presented in each column indicates that the leading factor of the corresponding measure (RM or SRM) matches the lower bound.
}
\label{tbl_comp}
\begin{tabular}{lccc} 
& RM & SRM \\ \hline
Frequentist & 
\begin{tabular}{@{}c@{}}
$\Theta(\log T)$\\ 
\citet{Lairobbins1985} 
\end{tabular}
& 
\begin{tabular}{@{}c@{}}
$\Theta(\exp(-\Gamma T))$\\ 
\citet{Carpentier2016} (up to a constant)
\end{tabular}
\\ 
\\
Bayesian & 
\begin{tabular}{@{}c@{}}
$\Theta((\log T)^2)$\\ 
\citet{lai1987}
\end{tabular}
& 
\begin{tabular}{@{}c@{}}
$\Theta(T^{-1})$\\
This paper (up to a constant)
\end{tabular}
\\
\hline
\end{tabular}
\end{center}
\end{table}

According to Theorem 3 of \citet{lai1987}, an asymptotically optimal algorithm's Bayesian regret is \begin{equation}\label{ineq_laibound}
\Regret_H(T) := \Ex_{\bmu \sim H}[ \Regret_{\bmu}(T) ]
=
\frac{(\log T)^2}{2} \sum_{i=1}^K \int_{ [0,1]^{K-1} } \hcond{i}(\mustarex{i}| \bmu_{\setminus i}) dH_{\setminus_i}(\bmu_{\setminus i})
+ o((\log T)^2)
,
\end{equation}
where $\mustarex{i} = \max_{j \ne i} \mu_j$.

This paper shows that the expected Bayesian simple regret is bounded as
\begin{equation}
\SRegBayes(T) \le \frac{1}{T} \sum_{i=1}^K
\int_{ [0,1]^{K-1} }
\mustarex{i}(1-\mustarex{i}) 
\hcond{i}(\mustarex{i}|\bmuex{i}) 
d\Hex{i}(\bmu_{\setminus i}) + o\left(\frac{1}{T}\right)
\label{ineq_optbound_informal}
\end{equation}
and derives the corresponding lower bound that matches the upper bound up to a (universal) constant factor.
That is, we characterize the optimal rate of the Bayesian simple regret under the continuity assumption of the prior.

Among the greatest challenges for establishing the Bayesian simple regret bound is the absence of any notion for characterizing ``good'' algorithm.
In the case of regret minimization (RM), \citet{Lairobbins1985} proposed a notion of ``uniformly good''; an algorithm is uniformly good if it has $o(T^c)$ regret for any $c>0$ and for any fixed parameter $\bmu$. 
Almost all meaningful algorithms in RM setting are uniformly good.\footnote{For example, $\epsilon_t$-greedy \citep{auer2002}, Upper Confidence Bound \citep{Lairobbins1985,auer2002}, Thompson sampling \citep{Thompson1933}, and Minimum Empirical Divergence \citep{honda2015} are uniformly good.} 
The bound of Eq.~\eqref{ineq_laibound} can be explained by (some of) the asymptotically optimal algorithms among the uniformly good algorithms.
In contrast, an optimal algorithm in the context of Bayesian simple regret minimization (SRM) remains relatively unexplored.
Although certain frequentist characterizations are known \citep{Audibert10,Carpentier2016}, there are currently no notions that correspond to RM's notions of ``uniformly good'' or ``asymptotical optimality.''
Accordingly, this paper demonstrates that a minimal assumption on the prior distribution can be sufficient to derive the asymptotic rate of the Bayesian simple regret.

\subsection{Intuitive derivation of the bound}
\label{subsec_informal}

This section provides an informal derivation of Eq.~\eqref{ineq_optbound_informal}. 
Section \ref{sec_alg} presents the formal results.

Consider the parameters $\bmu: \mu_i > \mustarex{i}$ where arm $i$ is the best arm.
The Kullback-Leibler (KL) divergence between parameters $(\mu_i, \mustarex{i})$ and $(\frac{\mustarex{i}+\mu_i}{2}, \frac{\mustarex{i}+\mu_i}{2})$ characterizes the difficulty of confirming that $\mu_i$ is larger than $\mustarex{i}$. That is, the frequentist simple regret for parameter $\bmu$ is approximately
\begin{equation}\label{ineq_roughregret}
\frac{\mu_i - \mustarex{i}}{2} \exp\left(-T
\KL\left(\mustarex{i},\frac{\mustarex{i}+\mu_i}{2}\right)
\right)
,
\end{equation}
where $\KL(p,q) = p \log(p/q) + (1-p) \log((1-p)/(1-q))$ is the KL divergence between two Bernoulli distributions with parameters $p,q \in (0,1)$.
Integrating this over the prior yields
\begin{align}
\lefteqn{
\int_{ [0,1]^{K-1} } \int_{(\mustarex{i}, 1]} 
\frac{\mu_i - \mustarex{i}}{2} \exp\left(-T
\KL\left(\mustarex{i},\frac{\mustarex{i}+\mu_i}{2}\right)
\right) 
h_i(\mu_i|\bmu_{\setminus i})d\mu_i dH_{\setminus_i}(\bmu_{\setminus i})
}\\
&\approx 
\int_{ [0,1]^{K-1} } \int_{(\mustarex{i}, \mustarex{i}+O(\sqrt{\frac{\log T}{T}})]} 
\frac{\mu_i - \mustarex{i}}{2} \exp\left(-T
\KL\left(\mustarex{i},\frac{\mustarex{i}+\mu_i}{2}\right)
\right) 
h_i(\mustarex{i}|\bmu_{\setminus i})d\mu_i dH_{\setminus_i}(\bmu_{\setminus i})\\
&\text{\ \ \ \ \ (continuity implies $h_i(\mu_i|\bmu_{\setminus i})\approx h_i(\mustarex{i}|\bmu_{\setminus i})$ when $\mu_i - \mustarex{i}=O(\sqrt{(\log T)/T})=o(1)$)
}\\
&\approx 
\int_{ [0,1]^{K-1} } 
\left[ 
-\frac{2\mustarex{i}(1-\mustarex{i})}{T} 
\exp\left(-T
\KL\left(\mustarex{i},\frac{\mustarex{i}+\mu_i}{2}\right)
\right)
\right]_{\mustarex{i}}^{\mustarex{i}+O(\sqrt{\frac{\log T}{T}})}
h_i(\mustarex{i}|\bmu_{\setminus i})
dH_{\setminus_i}(\bmu_{\setminus i}) \\
&=
\int_{ [0,1]^{K-1} } 
\frac{2\mustarex{i}(1-\mustarex{i})}{T} 
h_i(\mustarex{i}|\bmu_{\setminus i})
dH_{\setminus_i}(\bmu_{\setminus i})
+o\left(\frac{1}{T}\right)
,\label{ineq_twicebound}
\end{align}
which is twice as Eq.~\eqref{ineq_optbound_informal}. A more elaborated analysis removes the factor of two and yields Eq.~\eqref{ineq_optbound_informal}.
The derivation implies
\begin{itemize}
\item The region that matters is $\mu_i - \mustarex{i} = O(1/\sqrt{T}) = o(\sqrt{(\log T)/T})$, which is small when $T$ is large. By Assumption \ref{assp_cont}, $h_i(\mu_i|\bmuex{i})$ is sufficiently flat in this region.
\item Let $\jst$ be the second best arm.
When $|\mu_i - \mustarex{i}| \approx O(1/\sqrt{T})$ and the arms other than $i,\jst$ are substantially suboptimal, the optimal strategy invests most of the $T$ rounds into the two arms. Eq.~\eqref{ineq_roughregret} represents the information-theoretic bound of identifying the parameters $(\mu_i, \mustarex{i})$ (i.e., arm $i$ is better) from $(\frac{\mustarex{i}+\mu_i}{2}, \frac{\mustarex{i}+\mu_i}{2})$ (i.e., both arms are the same).
\item %
The term $\mustarex{i}(1-\mustarex{i})$ implies that the closer the best arm to $0,1$, the more identifiable it is. 
This is intuitive because the KL divergence diverges around $0,1$ in Bernoulli distributions.
\end{itemize}

\subsection{Related work}

The multi-armed bandit (MAB) problem
has garnered much attention in the machine learning community because it is useful in several crucial applications such as online advertisements and A/B testings. The goal of the standard MAB problem is to maximize the sum of the rewards, which boils down to regret minimization (RM). On the other hand, there is another established branch of bandit problems, called best arm identification (BAI, \cite{Audibert10}).
In BAI, the goal is to find the arm with the highest expected reward; this is closely related to the classical sequential tests \citep{Chernoff1959}. 
Maximizing the mean quality of the recommendation arm $J(T)$ boils down to simple regret minimization (SRM).
The different goals of RM and SRM mean that algorithms differ considerably in terms of balancing exploration and exploitation.

\noindent\textbf{Best arm identification (simple regret minimization)} 
Although the term ``best arm identification'' was coined in early 2010s \citep{Audibert10,Bubeck2011}, similar ideas have attracted substantial attention in various fields \citep{Paulson1964,Maron1997,EvanDar2006}.
There are two main settings in BAI. In the fixed-confidence setting, the goal is to minimize the number of samples required to control the probability of error (PoE, $\Prob[J(T) \ne i^*]$) the best arm to a pre-specified value. In the fixed-budget setting, the objective is to maximize the quality of the estimated best arm given a fixed number of samples. In this paper, we focus on the fixed-budget setting.
For this setting, \citet{Audibert10} proposed the successive rejects algorithm, which has frequentist simple regret of the order $\exp(- \Gamma T)$. \citet{Carpentier2016} showed an example where $\Gamma$ is optimal up to a constant factor.
Sometimes, PoE, which also has the same exponential rate as the frequentist simple regret (\cite{Audibert10}, Section 2), is used to measure a fixed-budget BAI algorithm \cite{komiyama2022minimax}.
The difference between the simple regret and PoE matters in terms of rate when considering a Bayesian objective, unlike in the frequentist case.

\noindent\textbf{Ordinal optimization (ranking and selection):} 
A particularly interesting strand of literature concerns the ordinal optimization \citep{ho1992,chen2000}, for which \cite{glynn2004large} provides a rigorous modern foundation. 
Although ordinal optimization and BAI are both interested in finding optimal arms, the two frameworks differ markedly. The framework of \citet{glynn2004large} assumes that the model parameters $\bmu$ are known, BAI assumes that the parameters are unknown. In practice, these parameters are often unknown, necessitating the use of plug-in estimators. However, the convergence of the plug-in estimators to the true parameters matters in fixed-budget setting \citep{Carpentier2016,komiyama2022minimax}. 
The objective that is essentially equivalent to PoE is studied as the probability of correct selection (PCS) in this literature. In particular, \citet{Yijie2016} studied a consistent algorithm in view of Bayesian version of PCS. However, they did not derive the rate of the objective. \citet{Hong2021review} review the development of this topic from the 1950s up to the 2020s. They pointed out the difference between BAI and the ordinal optimization (ranking and selection) as ``BAI problem assumes the samples to be bounded or sub-Gaussian distributed, whereas the ranking and selection problem typically assumes they are Gaussian distributed with unknown variances.'' In this sense, this paper considers Bernoulli rewards, and belongs to the former category.

\noindent\textbf{Bayesian algorithms for regret minimization:}
Thompson sampling \citep{Thompson1933}, among the oldest heuristics, is known to be asymptotically optimal in terms of the frequentist regret \citep{Granmo2008,agrawal2012,kaufmann2012}. One of the seminal results regarding Bayesian regret is the Gittins index theorem \citep{Gittins89,weber1992gittins}, which states that minimizing the discounted Bayesian regret is achieved by computing the Gittins index of each arm.
However, the Gittins index is no longer optimal in the context of undiscounted regret. Note also that there are some similarities between the frequentist method and the Gittins index \citep{russo2021}.

\noindent\textbf{Bayesian algorithms for simple regret minimization:} 
Regarding the objectives related to the identification of the best arm (i.e., SRM or PoE minimization), \citet{Russo2016} presented a version of Thompson sampling and derived its posterior convergence in a frequentist sense. Elsewhere, \citet{Xuedong2020} extended the algorithm of \citet{Russo2016} to demonstrate asymptotic optimality in the sense of the frequentist lower bound for the fixed-confidence setting. 
The expected improvement algorithm, a well-known myopic heuristic, is known to be suboptimal in the context of SRM \citep{Ryzhov2016}. However, \citet{Qin2017} demonstrated that a modification of the algorithm enables good posterior convergence.
Note that the Bayesian algorithms discussed have been evaluated in terms of frequentist simple regret or posterior convergence; that is, the scholarship includes a limited discussion of the Bayesian simple regret. 
\citet{Russo2017IDS} and \citet{Qin2022} considered variants of information-directed sampling and Thompson sampling. They derived $O(1/\sqrt{T})$ Bayesian simple regret bounds (Section 9.1 of \citet{Russo2017IDS}).

\noindent\textbf{Gaussian process bandits:} Finally, it is necessary to introduce Gaussian process bandits, also known as Bayesian optimization \citep{frazier2018}.
While Gaussian process bandits originally aimed to minimize the Bayesian simple regret, seminal papers have analyzed a worst-case (minimax) simple regret \citep{bull2011} or high-probability bound for simple regret \citep{srinivas2010,vakili2021}.

\section{Proposed Algorithm: Two-Stage Exploration}\label{sec_alg}

Bayesian simple regret (i.e., Eq.~\eqref{ineq_reg_bayes})
is exactly minimized by solving the corresponding dynamic programming. However, computing such dynamic programming does not scale for moderate $K$ and $T$. The number of possible states characterizes the amount of computation required. 
In the Bernoulli MAB problem, the number of possible states is proportional to the number of rewards $0$ and $1$ for each arm, which is $O(T^{(2K-1)})$.

\subsection{Two-stage exploration procedure}

\begin{algorithm}[!t]
\caption{Two-Stage Exploration Algorithm}
\label{alg_proposed}
\begin{algorithmic}
    \Require $q \in (0,1), T \in \Natural$.
    \State Draw each arm $qT/K$ times.
    \State 
    At the end of round $qT$, for each arm $i \in [K]$, calculate the lower and upper confidence bounds \[
    L_i = \hatmu_i(qT) - \Conf \quad\text{and}\quad U_i = \hatmu_i(qT) + \Conf.
    \]
    \State Compute a set of candidates $\bestcand := \{i\in[K]: U_i \ge \max_j L_j\}$. 
    \If{$|\bestcand| = 1$}
        \State Immediately return the unique arm in $\bestcand$.
    \Else
        \State Draw each arm in $\bestcand$ for $(1-q)T/|\bestcand|$ times.
        \State Return $J(T) = \argmax_{i \in \bestcand} \hatmu_i(T)$.
    \EndIf
    \end{algorithmic}
\end{algorithm}

Instead of the computationally prohibitive dynamic programming procedure, we introduce the two-stage exploration (TSE) algorithm (Algorithm \ref{alg_proposed}), which requires only summary statistics, which are easily computed. 
The TSE algorithm conducts uniform exploration during the first $qT$ rounds. Based on the empirical means $\hatmu_i(qT)$ at the end of round $qT$, it identifies a set of best arm candidates $\bestcand$. Using the confidence bound of width 
\[ 
\Conf(T) = \sqrt{\frac{K \log T}{qT}},
\]
the true best arm is found in $\bestcand$ with high probability. The remaining $(1-q)T$ rounds are exclusively dedicated to the arms in $\bestcand$. Following round $T$, the TSE algorithm recommends the arm with the largest empirical mean among $\bestcand$.

TSE is an elimination algorithm that maintains a list of the best arm candidates and progressively narrows it. It differs from popular alternatives, such as successive rejects \cite{Audibert10} and sequential halving \cite{Shahin2017}, and is optimized to minimize the Bayesian simple regret, although it is a frequentist algorithm that does not require a prior.
According to Section \ref{subsec_informal}, the leading term of the Bayesian simple regret stems from the case $\mu_i - \mustarex{i} = o(\sqrt{(\log T)/T})$ and $\mu_i - \mu_j = \Omega(\sqrt{(\log T)/T})$ for arm $j \notin \{i, \jst\}$. In this case, we can quickly eliminate the arms other than $\{i, \jst\}$ by choosing a small $q>0$ and can spend most of the samples to $\{i, \jst\}$. Although successive rejects and sequential halving are expected to have $O(1/T)$ Bayesian simple regret like TSE, their respective constant factors are suboptimal due to their substantial sample allocation to arms $j \notin \{i, \jst\}$.

\subsection{Regret analysis of two-stage exploration}

The following theorem provides a simple regret guarantee of the proposed TSE algorithm.
\begin{thm}{\rm (Bayesian Simple regret upper bound of TSE)}\label{thm_regupper}
Let $\Tdash := 2qT/K + (1-q)T$.
Under Assumption \ref{assp_cont}, for any $q>0$, the Bayesian simple regret of Algorithm \ref{alg_proposed} is bounded as follows:
\begin{equation}
\SRegBayes(T) \le \frac{\Cbayesopt}{\Tdash} + o\left(\frac{1}{T}\right),
\end{equation}
where
\begin{equation}
\Cbayesopt = \sum_{i \in [K]} \int_{[0,1]^{K-1}} 
\mustarex{i}(1-\mustarex{i}) h_i(\mustarex{i}|\bmuex{i}) d\Hex{i}(\bmuex{i}).
\end{equation}
\end{thm}

\begin{remark}{\rm (Hyperparameter $q$)}
Theorem \ref{thm_regupper} with small $q \rightarrow 0$ implies
\begin{equation}
\SRegBayes(T) \le (1+\eps)\frac{\Cbayesopt}{T} + o\left(\frac{1}{T}\right),
\end{equation}
for any $\eps>0$ because $\Tdash \rightarrow T$ as $q \rightarrow 0$.
\end{remark}

We use the following lemmas to derive Theorem \ref{thm_regupper}.
The proofs of all lemmas are in the appendix.
For two events $\EA$ and $\EB$, let $\{\EA, \EB\} = \EA \cap \EB$.
Let us define an event for which the true parameters lie within the confidence bounds as
\begin{equation}\label{ineq_inbound}
\ES = \bigcap_{i \in [K]} \{L_i \le \mu_i \le U_i\}.
\end{equation}
Let $\Deltast = \mu^* - \mu_{J(T)}$ be the loss of recommending arm $J(T)$. 
\begin{lem}{\rm ($\mathcal{S}$ occurs with high-probability)}\label{lem_inbound}
\begin{equation}
\Prob[\ES] \ge 1 - \frac{2K}{T^2}.
\end{equation}
\end{lem}
\begin{lem}\label{lem_threenearalg}
\begin{equation}
\Ex_{\bmu \sim H}\left[ \Ex_{\bmu}\left[\Ind\left[\ES, \left|\bestcand\right|\ge3\right]\Deltast\right] \right] = 
o\left(\frac{1}{T}\right),
\end{equation}
where $\Ind[\EA] = 1$ if $\EA$ holds or $0$ otherwise.
\end{lem}
\begin{lem}\label{lem_besttwo}
The following inequality holds:
\begin{equation}
\Ex_{\bmu \sim H}\left[ \Ex_{\bmu}\left[\Ind\left[\ES, \left|\bestcand\right|=2\right] \Deltast\right] \right] 
\le
\frac{\Cbayesopt}{\Tdash} + o\left(\frac{1}{T}\right).
\end{equation}
\end{lem}
Lemma~\ref{lem_inbound} states that the true parameters lie in the confidence bounds with high probability. 
Lemma~\ref{lem_threenearalg} states that the case of $|\bestcand|\ge 3$ is negligible with a large $T$, and Lemma \ref{lem_besttwo} states the leading factor stems from the case of $|\bestcand|=2$. 
\proof{Proof of Theorem \ref{thm_regupper}}
The simple regret of TSE is bounded as 
\begin{align}
\SRegBayes(T) 
&:= \Ex_{\bmu \sim H}[ \SRegFreq(T) ]\\
&\le \Ex_{\bmu \sim H}\left[ \Ex_{\bmu}[\Ind[\ES] \Deltast] \right] + \frac{2K}{T^2} 
\text{\ \ \ (by Lemma \ref{lem_inbound})}\\
&= 
\Ex_{\bmu \sim H}\left[ \Ex_{\bmu}[\Ind[\ES,|\bestcand|=1]  \Deltast] \right] +
\Ex_{\bmu \sim H}\left[ \Ex_{\bmu}[\Ind[\ES,|\bestcand|=2]  \Deltast] \right] \\
&\ \ \ \ 
+ \Ex_{\bmu \sim H}\left[ \Ex_{\bmu}[\Ind[\ES,|\bestcand|\ge3]  \Deltast] \right] 
+ \frac{2K}{T^2} \\
&= 
\Ex_{\bmu \sim H}\left[ \Ex_{\bmu}[\Ind[\ES,|\bestcand|=2] \Deltast] 
\right] + \Ex_{\bmu \sim H}\left[ \Ex_{\bmu}[\Ind[\ES,|\bestcand|\ge3] \Deltast] 
\right] 
+ \frac{2K}{T^2} \\
&\text{\ \ \ \ (by $\ES$ implies $\ist \in \bestcand$)}
\\
&= 
\Ex_{\bmu \sim H}\left[ \Ex_{\bmu}[\Ind[\ES,|\bestcand|=2] \Deltast] \right] + o\left(\frac{1}{T}\right) + \frac{2K}{T^2} 
\text{\ \ \ (by Lemma \ref{lem_threenearalg})} 
\\
&\le
\frac{\Cbayesopt}{T'} + o\left(\frac{1}{T}\right) + \frac{2K}{T^2}.
\text{\ \ \ (by Lemma \ref{lem_besttwo})} 
\end{align}
\endproof

\begin{example}{\rm (Uniform prior)}\label{expl_product}
In the case of the uniform prior where each $\mu_i$ is independently drawn from $\Unif(0,1)$, 
\begin{equation}
h_i(\mu_i|\bmuex{i}) = 1,
\end{equation}
and $\mustarex{i}$ is distributed with its cdf $(\mustarex{i})^{K-1}$. Using this, we can obtain
\begin{align}
\Cbayesopt 
&= \sum_{i \in [K]} \int_{[0,1]}
\mustarex{i}(1-\mustarex{i}) \times (K-1)(\mustarex{i})^{K-2} d \mustarex{i}\\
&= \frac{K-1}{K+1}.
\end{align}
Section \ref{sec_sim} confirms that the empirical performance of the TSE algorithm matches this constant.
\end{example}
More generally, for an i.i.d. prior, we can use the fact that the cdf of $\mustarex{i}$ is $(H_j(\mu_j))^{K-1}$, where $H_j$ is the corresponding cdf of asingle arm. 

\section{Lower Bound of Bayesian Simple Regret}
\label{sec_lower}

\subsection{Lower bound}

The following theorem characterizes the achievable performance of any algorithm.
\begin{thm}{\rm (Simple regret lower bound of an arbitrary algorithm)}\label{thm_reglower}
Under Assumption \ref{assp_cont}, for any BAI algorithm, we have
\begin{equation}
\SRegBayes(T) \ge \frac{\Cbayesopt}{4.8T} - o\left(\frac{1}{T}\right),
\end{equation}
where $\Cbayesopt$ is the constant that is defined in Theorem \ref{thm_regupper}.\footnote{
Since $o(1)$ is a function $f(T)$ such that $\lim_{T \rightarrow \infty} |f(T)|=0$, $+o(1)$ and $-o(1)$ are the same. 
For clarity, we use $+o(1)$ for upper bounds and $-o(1)$ for lower bounds, respectively.
}
\end{thm}
Theorem \ref{thm_reglower} states that TSE is optimal up to a constant factor. 
For any prior distribution,\footnote{The prior distribution of the arms can be correlated as long as Assumption \ref{assp_cont} holds.} no algorithm,\footnote{Regardless of knowledge of the prior.} has a smaller order of simple regret than the TSE algorithm. Namely,
\begin{equation}
\limsup_{T \rightarrow \infty} \frac{\SRegBayes^{\mathrm{TSE}}(T)}{\SRegBayes^{\mathrm{\EA^*}}(T)}
\le
4.8,
\end{equation}
where $\SRegBayes^{\mathrm{TSE}}(T)$ is the simple regret of TSE and $\SRegBayes^{\mathrm{\EA^*}}(T)$ is the simple regret of the optimal algorithm, which is to solve  a dynamic programming at each round.

The rest of this section derives Theorem \ref{thm_reglower}, which requires the introduction of some notation and several lemmas.
Let
\begin{align}
\Theta_i &= \{\bmu: \mu_i > \max_{j \ne i} \mu_j\} \\
\Theta_{i,\marg} &= \{\bmu: \mu_i > \max_{j \ne i} \mu_j,\ \mu_i \in [T^{-\marg}, 1-T^{-\marg}]\} \\
\Theta_{i,j,\marg} &= \{\bmu \in \Theta_{i,\marg}: \mu_j + 2 \Conf(T) > \mu_i > \mu_j > \mustarex{ij}\}, \end{align}
where $\mustarex{ij} = \max_{k \ne \{i,j\}} \mu_k$.
Namely, $\Theta_i$ is the parameter set where $i$ is the best arm, $\Theta_{i,\marg}$ is a subset of $\Theta_i$ where $\mu_i$ is not very close\footnote{The set $\Theta_{i,\marg}$ is introduced to avoid very large value of the KL divergence around $\mu_i \approx 0, 1$.} to $0, 1$. 
Moreover, $\Theta_{i,j,\marg}$ is a subset of $\Theta_{i,\marg}$ where $j$ is the second-best arm, such that $|\mu_i - \mu_j|$ is very small.\footnote{Remember that $\Conf(T) = O(\sqrt{(\log T)/T})$.} Simple regret is characterized by this region.
We use the following Lemmas to prove Theorem~\ref{thm_reglower}.

\begin{lem}{\rm (Exchangeable mass)}\label{lem_exchange}
Let $f(\mu_i, \mu_j, \bmuex{ij})$ be any function on $[0,1]$. 
Then,
\begin{equation}
\int_{\Theta_{i,j,\marg}} 
(\mu_i - \mu_j)
 f(\mu_i,\mu_j,\bmuex{ij}) dH(\bmu)
=
(1-o(1))\int_{\Theta_{j,i,\marg}} 
(\mu_j - \mu_i)
 f(\mu_j,\mu_i,\bmuex{ij}) dH(\bmu).
\end{equation}
\end{lem}

\begin{lem}\label{lem_prooflb_kcg}
For any $\eta > 0$, there exists $T_0$
such that the following inequality holds for all $T \ge T_0, \bmu \in \Theta_{i,j,\marg}$: 
\begin{equation}\label{ineq_prooflb_kcg}
\max(\Prob_{\bmu}[J(T) \ne i], \Prob_{\bnu}[J(T) \ne j]) \ge 
\frac{1}{2.4} \exp\Big( -(1+\eta)T\KL(\mu_j, \mu_i) \Big),
\end{equation}
where 
\[
\bnu := (\mu_1, \mu_2, \dots, \mu_{i-1},\underbrace{\mu_j}_{\text{$i$-th element}},\mu_{i+1},\dots, \mu_{j-1},\underbrace{\mu_i}_{\text{$j$-th element}},\mu_{j+1}, \dots, \mu_K) \in \Theta_{j,i,\marg}
\]
be another set of parameters, such that $(\mu_i, \mu_j)$ are swapped from $\bmu$. 
\end{lem}

\begin{lem}{\rm (Integration on the lower bound)}\label{lem_integration}
The following equality holds:
\begin{align}
\int_{[\mustarex{i}, \mustarex{i}+2\Conf]} (\mu_i - \mustarex{i}) 
\exp\left( -(1+\eta)T\KL(\mustarex{i}, \mu_i) \right)
d\mu_i 
=
\frac{\mustarex{i}(1-\mustarex{i})}{(1+\eta)T} - o(1).
\end{align}
\end{lem}
Lemma~\ref{lem_exchange} states that the area of $\Theta_{i,j,\marg}$ and $\Theta_{j,i,\marg}$ are approximately equal.\footnote{Lemma~\ref{lem_exchange} is placed to absorb the difference between $h_{ij}(m+\delta, m|\bmuex{ij})$ and $h_{ij}(m, m+\delta|\bmuex{ij})$. This lemma is unnecessary for a symmetric model such as the uniform prior (Example~\ref{expl_product}).
}
Lemma~\ref{lem_prooflb_kcg}, which utilizes Lemma~1 of \citet{kaufman16a}, represents the performance tradeoff between identifying $J(T)=i$ and $J(T)=j$.
Lemma \ref{lem_integration} integrates frequentist simple regret over the conditional distribution of $\mu_i$ given $\bmuex{i}$.

\proof{Proof of Theorem \ref{thm_reglower}}
By definition, 
\begin{align}
\SRegBayes(T)= \int_{[0,1]^K} \SRegFreq(T) dH(\bmu) = \int_{[0,1]^K} \Big(\mu^* - \Ex_{\bmu}[\mu_{J(T)}]\Big) dH(\bmu).
\end{align}
We have,
\begin{align}
&\int_{[0,1]^K} \Big(\mu^* - \Ex_{\bmu}[\mu_{J(T)}]\Big) dH(\bmu)\\
&= \sum_{i \in [K]} \int_{[0,1]^K} \Ind[\bmu \in \Theta_i] (\mu_i - \mustarex{i}) \Prob_{\bmu}[J(T) \ne i] dH(\bmu) \\
&\ge \sum_{i \in [K]} \sum_{j\ne i} \int_{[0,1]^K} \Ind[\bmu \in \Theta_{i,j,\marg}] (\mu_i - \mu_j) \Prob_{\bmu}[J(T) \ne i] dH(\bmu) \\ 
&= (1-o(1))\sum_{i \in [K]} \sum_{j\ne i} \int_{[0,1]^K} \Ind[\bmu \in \Theta_{i,j,\marg}] (\mu_i - \mu_j) \frac{\Prob_{\bmu}[J(T) \ne i] + \Prob_{\bnu}[J(T) \ne j]}{2} dH(\bmu) 
\\ &\text{\ \ \ \ \ (by Lemma \ref{lem_exchange} with $f = \Prob_{\bmu}[J(T) \ne i]$)}\\
&\ge \frac{(1-o(1))}{4.8} \sum_{i \in [K]} \sum_{j\ne i} \int_{[0,1]^K} \Ind[\bmu \in \Theta_{i,j,\marg}] (\mu_i - \mu_j) 
\exp\left( -(1+\eta)T\KL(\mu_j, \mu_i) \right)
dH(\bmu) %
\\ &\text{\ \ \ \ \ (by Lemma \ref{lem_prooflb_kcg})}\\
&= \frac{(1-o(1))}{4.8} \sum_{i \in [K]} \int_{[0,1]^{K-1}} \int_{[0,1]} \Ind[\bmu \in \Theta_{i,\marg}] (\mu_i - \mustarex{i}) 
\exp\left( -(1+\eta)T\KL(\mustarex{i}, \mu_i) \right)
dH_i(\mu_i) d\Hex{i}(\bmuex{i})
\\ 
&= \frac{(1-o(1))}{4.8}\sum_{i \in [K]}  \int_{[0,1]^{K-1}} h_i(\mustarex{i}|\bmuex{i}) 
\left(
\int_{[\mustarex{i},\mustarex{i}+2\Conf]} (\mu_i - \mustarex{i}) 
\exp\left( -(1+\eta)T\KL(\mustarex{i}, \mu_i) \right)
d\mu_i 
\right)
d\Hex{i}(\bmuex{i}) 
\\
&\text{\ \ \ \ \ (by uniform continuity)}\\
&=\frac{(1-o(1))}{4.8(1+\eta)T}\sum_{i \in [K]} \int_{[0,1]^{K-1}} 
\mustarex{i}(1-\mustarex{i}) h_i(\mustarex{i}|\bmuex{i}) d\Hex{i}(\bmuex{i})
,\\
&\text{\ \ \ \ \ (by Lemma \ref{lem_integration})}
\end{align}
where the above inequality holds for any $\eta>0$.
Here, in the transformation from the third line to the fourth line, we applied Lemma \ref{lem_exchange} to the half of the quantity, which yields
\begin{align}
\lefteqn{
\int_{\Theta_{i,j,\marg}} (\mu_i - \mu_j) \Prob_{\bmu}[J(T) \ne i] dH(\bmu)
}\\
&=
\frac{1-o(1)}{2}
\left(
\int_{\Theta_{i,j,\marg}} (\mu_i - \mu_j) \Prob_{\bmu}[J(T) \ne i] dH(\bmu)
+
\int_{\Theta_{j,i,\marg}} (\mu_j - \mu_i) \Prob_{\bnu}[J(T) \ne i] dH(\bmu)
\right), 
\end{align}
and thus 
\begin{align}
\lefteqn{
\int_{\Theta_{i,j,\marg}} (\mu_i - \mu_j) \Prob_{\bmu}[J(T) \ne i] dH(\bmu)
+
\int_{\Theta_{j,i,\marg}} (\mu_j - \mu_i) \Prob_{\bnu}[J(T) \ne j] dH(\bmu)
}\\
&=(1-o(1))\Biggl(
\int_{\Theta_{i,j,\marg}} (\mu_i - \mu_j) \frac{\Prob_{\bmu}[J(T) \ne i] + \Prob_{\bnu}[J(T) \ne j]}{2} dH(\bmu)\\
&\hspace{7em}+
\int_{\Theta_{j,i,\marg}} (\mu_j - \mu_i) \frac{\Prob_{\bmu}[J(T) \ne j] + \Prob_{\bnu}[J(T) \ne i]}{2}
dH(\bmu)
\Biggr), 
\end{align}
for each pair $i,j$. 
The proof is completed.
\endproof %

\begin{remark}{\rm (Finite-time analysis)}
Theorems \ref{thm_regupper} and \ref{thm_reglower} are asymptotic. The only point at which we lose the finite-time property is on the continuity of the conditional cumulative distribution function (cdf) $h_i(\mu_i|\bmuex{i})$: We do not specify how fast $h_i(\mu_i|\bmuex{i})$ changes as a function of $\mu_i$. 
It is not very difficult to derive a finite-time bound
for specific models where the sensitivity of $h_i(\mu_i|\bmuex{he i})$ is known, as in the case of the uniform prior of Example \ref{expl_product}, where $h_i(\mu_i|\bmuex{i}) = 1$. 
In this case, the $o(1)$ term in Theorems~\ref{thm_regupper} and \ref{thm_reglower} can be replaced by a factor $C \Tmargconfrate$ for some constant $C>0$. 
\end{remark}

\subsection{Towards a tight bound}
\label{subsec_tightchallenge}

Although there is a constant factor in the lower bound (i.e., $1/4.8$),
we hypothesize the upper bound (Theorem \ref{thm_regupper}) is tight for the following reasons.
\begin{itemize}
\item The TSE algorithm only spends $q$ fraction of rounds identifying $\bestcand$, and we may set small $q$ with a large $T$. When $|\bestcand|=2$ (which is the case that matters), it spends approximately $T/2$ rounds on each candidate of $\bestcand$; this appears to produce little to no space for improvements.
\item The simulation results, which we present in Section \ref{sec_sim}, empirically match the bound of Theorem~\ref{thm_regupper}.
\end{itemize}

Although we can increase factor $1/4.8$ to some extent, making it to $1$ is highly non-trivial.  
Let $i,j$ be the best two arms and $\mu_i - \mu_j = \Delta$. Among the largest challenges is the determination of the Bayesian simple regret in the region where $\Delta = O(\frac{1}{\sqrt{T}})$. In this case, $\KL(\mu_j, \mu_i) = O(\Delta^2)$ and
\begin{equation}
T \KL(\mu_j, \mu_i) = O(1).
\end{equation}
Meanwhile, the estimation error of $\mu_i$ with $T$ sample is $O(1/\sqrt{T})$; thus, 
\begin{equation}
T (\KL(\mu_j, \mu_i) - \KL(\mu_j, \hatmu_i))
\approx T \frac{\partial \KL(\mu_j, \mu_i)}{\partial \mu_i} (\mu_i - \hatmu_i) = T \times O\left(\frac{1}{\sqrt{T}}\right) \times O\left(\frac{1}{\sqrt{T}}\right) = O(1) 
\end{equation}
matters here. 
Unlike the frequentist case, a high-probability bound of the form $T \KL(\mu_j, \mu_i)$ is unavailable for deriving an optimal Bayesian simple regret lower bound.  

\section{Simulation}
\label{sec_sim}

\begin{figure}[t!]
    \centering
    \begin{minipage}[t]{0.48\textwidth}
         \centering
         \includegraphics[bb=0 0 399 266,width=\textwidth]{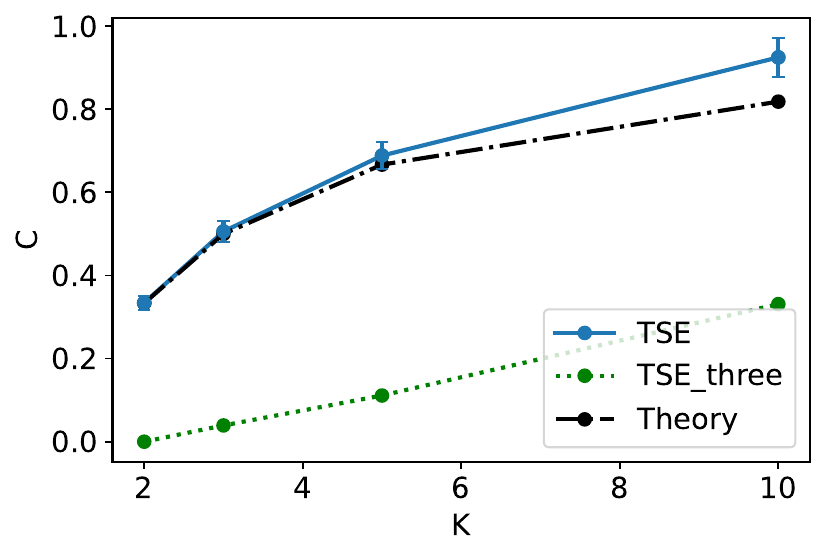}
         \caption{Comparison of $\Cbayesopt$ ($= (K-1)/(K+1)$, Theory) and its estimated value using the TSE algorithm ($= \SRegBayes(T)T'$) with several different values of $K$. Here, the value of $T$ is set to $10^5 K$ for each $K$. TSE\_three indicates the simple regret when $|\bestcand |\ge 3$, with each bar representing two sigma of the corresponding plug-in variance.}
         \label{fig:regret_left}
    \end{minipage}
    \hspace{0.02\textwidth}
    \begin{minipage}[t]{0.48\textwidth}
         \centering
         \includegraphics[bb=0 0 399 266,width=\textwidth]{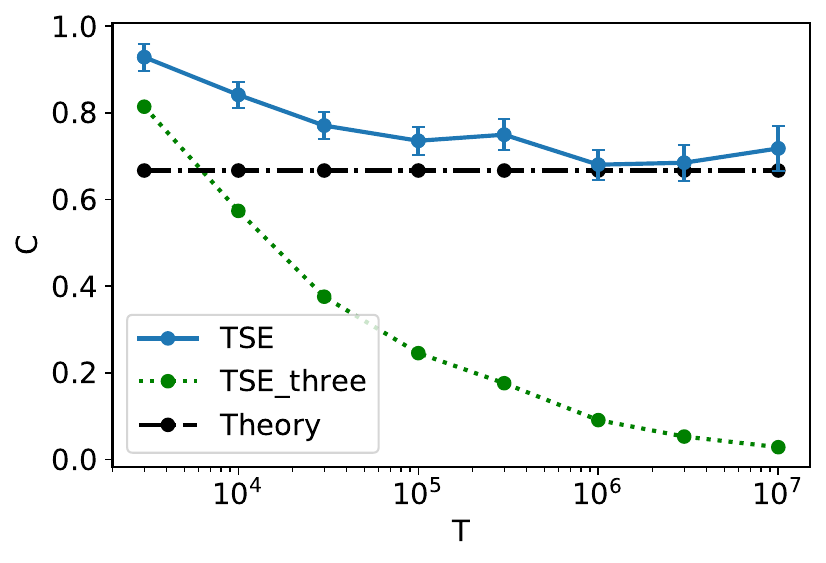}
         \caption{Comparison of $\Cbayesopt$ (Theory) and its estimated value using the TSE algorithm ($= \SRegBayes(T)T'$) with several different values $T$. Here, $K$ is fixed to be $5$. The decreasing value of TSE\_three implies that the probability of $|\bestcand| \ge 3$ is small if $T$ is large. Each bar represents two sigma of the corresponding plug-in variance.}
         \label{fig:regret_right}
    \end{minipage}
    \label{fig:regret}
\end{figure}

We conducted a set of simulations to support our theoretical findings. 
We empirically tested the TSE algorithm (Algorithm \ref{alg_proposed}) with the uniform prior (Example \ref{expl_product}) and measured its simple regret with the aim of verifying the tightness of the upper bound of Theorem~\ref{thm_regupper} with its leading order $\Cbayesopt/\Tdash$. 
We set $q=0.5$ in the simulations. 

To effectively sample small-gap cases, $\SRegBayes(T)$ should be computed by a Monte Carlo rejection sampling of the prior $\bmu$ with an acceptance ratio
\begin{equation}
\max\left(\min\left(\frac{1}{2\Delta\sqrt{T}}, 1\right), 0.01\right),
\end{equation}
where $\Delta$ is the gap between the best arm and the second-best arm.
Our simulation is implemented in the Python 3 programming language.\footnote{The source code of the simulation is available at \url{https://github.com/jkomiyama/bayesbai_paper}.}

Figure \ref{fig:regret_left} compares the performance of the TSE with the lower bound with several values of $K = \{2,3,5,10\}$. 
The results are averaged over $\Runnum$ runs.
The TSE performs very closely to $\Cbayesopt/T'$, which implies the tightness of Theorem \ref{thm_regupper}. 
The value $\SRegBayes(T)$ is slightly larger than $\Cbayesopt/T'$ because the case of $|\bestcand| \ge 3$ is non-negligible unless $T$ is infinitely large.
Furthermore, Figure \ref{fig:regret_right} sets $K=5$ and compares several different values of $T$. Larger $T$ values increase the probability of $|\bestcand| = 2$. Consequently, the simple regret of the TSE algorithm approaches the theoretical bound.

Note that this section does not aim to compare Algorithm \ref{alg_proposed} with existing algorithms. Although the TSE algorithm approaches the optimal bound by choosing a $q$ of $o(1)$, its practical performance has the capacity to improve.
The value $q=0.5$ splits the two stages evenly and does not compromise the asymptotic rate by more than a factor of two, and we consider it to be a reasonable choice in practice.

\section{Conclusion and Discussion}

We have analyzed the Bayesian BAI in the context of the $K$-armed MAB problem, which concerns identifying the Bayes-optimal treatment allocation.
We derived a lower bound of the Bayesian simple regret and introduced a simple algorithm that matches the lower bound under a mild regularity condition. 

Our result is a counterpart of the Bayesian regret minimization (RM) of \citet{lai1987} for the Bayesian simple regret minimization (SRM).
Upon deriving a lower bound for Bayesian simple regret, we introduced a simple algorithm to match the lower bound under a mild regularity condition. 
Our results constitute the counterpart of the Bayesian RM of \citet{lai1987} in the context of Bayesian SRM.
Several particular differences between RM and SRM can be summarized:
\begin{itemize}
\item In RM, certain asymptotically optimal algorithms adopting a frequentist perspective are also optimal in terms of Bayesian RM and vice versa,\footnote{See page 62 in \citet{kaufmannthesis}.} whereas optimality in frequentist SRM remains inadequately characterized even though a seminal paper by  \citet{Carpentier2016} exists.
\item The leading term of Bayesian RM derives from the region of $\Delta > \sqrt{\frac{\log T}{T}}$, whereas the leading term of Bayesian SRM derives from $\Delta < \sqrt{\frac{\log T}{T}}$.\footnote{Formally, $\int \SRegFreq(T) \Ind\left[\mu^*- \max_{j \ne \ist} \mu_j \ge \sqrt{(\log T)/T}\right]dH(\bmu) = o(1/T)$.} This makes analyzing Bayesian simple regret challenging (Section \ref{subsec_tightchallenge}). At the high level, Bayesian SRM involves quickly discarding clearly suboptimal arms and concentrating the rest of the samples on the competitive candidates of the best arm.
\end{itemize}
This paper has considered Bernoulli distributions, corresponding to observing a binary outcome for each treatment.
Extending our results to other distributions, such as Gaussian distributions, would represent an interesting future research direction.

\section*{Acknowledgments.}
We thank Po-An Wang and Kenshi Abe for several suggestions. 
We thank Assaf Zeevi for the discussion on the stability of algorithms.

\clearpage

\bibliographystyle{apalike}
\bibliography{references}

\clearpage

\appendix

\section{Supplemental Example}\label{sec_supplexpl}

This section shows an example in which the assumptions are violated.  
\begin{example}{\rm (A prior that depends on $T$)}\label{expl_twopoints}
Consider a two-armed case where the prior is \begin{itemize}
\item With probability $0.5$, $\mu_1 = T^{-1/2}$ and $\mu_2 = 0$.
\item With probability $0.5$, $\mu_1 = 0$ and $\mu_2 = T^{-1/2}$.
\end{itemize}
\end{example}
This prior violates the continuity assumption (Assumption \ref{assp_cont}) as well as it depends on $T$. It is not very difficult to show that the regret in this case is $\Theta(T^{-1/2})$, which is different from $O(T^{-1})$ regret that we have derived in the main paper.

\section{General Lemmas}

\begin{prop}{\rm (Stirling's approximation)}\label{prop_stirling}
For any $n \in \Natural$, the following inequality holds:
\begin{equation}
\sqrt{2\pi}n^{n+\frac{1}{2}}e^{-n}e^{\frac{1}{12n+1}}
<
n!
<
\sqrt{2\pi}n^{n+\frac{1}{2}}e^{-n}e^{\frac{1}{12n}}.
\end{equation}
By using $e^{1/n} \le 1 + 2/n$ for $n\ge2$, we have
\begin{equation}
\sqrt{2\pi}n^{n+\frac{1}{2}}e^{-n}
<
n!
<
\sqrt{2\pi}n^{n+\frac{1}{2}}e^{-n}\left(1+\frac{1}{6n}\right).
\end{equation}
\end{prop}

\begin{lem}{\rm (Chernoff bound)}\label{lem_chernoff}
Let $x_1,x_2,\dots,x_n$ be Bernoulli random variables with their common mean $\mu$. Let $\hatmu_n = (1/n)\sum_{t=1}^n x_t$. 
Then, 
\begin{align}\label{ineq_chernoff}
\Prob\left[\hatmu_n - \mu \ge \eps \right] &\le e^{-n\KL(\mu+\eps, \mu)}, \\
\Prob\left[\hatmu_n - \mu \le -\eps \right] &\le e^{-n\KL(\mu-\eps, \mu)}.
\end{align}
Moreover, using the Pinsker's inequality $\KL(p,q) \ge 2 (p-q)^2$ and taking a union bound yields
\begin{align}\label{ineq_hoeffding}
\Prob\left[|\hatmu_n - \mu| \ge \eps \right] &\le 2e^{-2n\eps^2}.
\end{align}
\end{lem}

\section{Bounds on KL Divergence}

Let $l \in [T^{-\marg}, 1-T^{-\marg}]$.
Let $\hTdecay = o(T^{-\marg})$, and $T_0 > 0$ be such that $2 \hTdecay < T^{-\marg}$ and $T^{-\marg} < 1/2$ for all $T \ge T_0$.
Let $u$ be such that $\hTdecay + l > u > l$.
Let $m = (l+u)/2$ and $\DeltaHalf = (u-l)/2$. 

Roughly speaking, Lemmas \ref{lem_klapprox} and \ref{lem_klsense}, which are below, state that 
\begin{equation}\label{ineq_roughklfinal}
\KL(l,u) \approx 4 \KL(l,m) \approx 4 \KL(m,u) 
\approx \frac{(u-l)^2}{2l(1-l)} \approx \frac{(u-l)^2}{2m(1-m)} \approx \frac{(u-l)^2}{2u(1-u)}
\end{equation}
when $\DeltaHalf$ is sufficiently small. 
We show the formal version of Eq.~\eqref{ineq_roughklfinal} in Lemma \ref{lem_klfinal}.

\begin{lem}{\rm (Bound on the KL divergence around $l$)}
\label{lem_klapprox}

For all $T \ge T_0$, the following inequality holds:
\begin{equation}
\left|
\KL(l,u) - \frac{(u-l)^2}{2l(1-l)}
\right|
\le
\CKL (\hTdecay)^3 T^{\margdbl},
\end{equation}
\end{lem}

\proof{Proof of Lemma \ref{lem_klapprox}}
Letting $\eta = \max(\frac{u-l}{l}, \frac{1-l}{1-u})
\le 4\hTdecay T^{\marg}$, we have
\begin{align}
\KL(l,u) 
&:= \int_{x=0}^{u-l} \frac{d(\KL(l,l+x))}{dx} dx\\
&= \int_{x=0}^{u-l} \frac{x}{(l+x)(1-l-x)} dx\\
&\le (1+2\eta)\int_{x=0}^{u-l} \frac{x}{l(1-l)} dx\\
&= (1+2\eta)\frac{(u-l)^2}{2l(1-l)}\\
&= \frac{(u-l)^2}{2l(1-l)} + 2\eta \times \frac{(u-l)^2}{2l(1-l)} = \frac{(u-l)^2}{2l(1-l)} + 2\eta \times \left((\hTdecay)^2 T^{\marg}\right)\\
&\le \frac{(u-l)^2}{2l(1-l)} + 8\left((\hTdecay)^3 T^{\margdbl}\right).
\end{align}
Another inequality $\frac{(u-l)^2}{2l(1-l)} - 8\left((\hTdecay)^3 T^{\margdbl}\right) \le \KL(l,u)$ is derived in the same manner. 
\endproof %

\begin{lem}\label{lem_klsense}
For all $T \ge T_0$, the following inequalities hold:
\begin{align}
|\KL(m, u) - \KL(l, m)| 
&\le \CKL (\hTdecay)^3 T^{\margdbl} \label{ineq_bothslide}\\
|\KL(m, u) - \KL(u, m)| 
&\le \CKL (\hTdecay)^3 T^{\marg},\label{ineq_symmetry}
\end{align}
for some universal constant $\CKL>0$.
\end{lem}
\proof{Proof of Lemma \ref{lem_klsense}}
\begin{align}
|\KL(m,u)-\KL(l,m)|
&= \left|\int_{x=l}^{m} \frac{d(\KL(x,x+\DeltaHalf))}{dx} dx\right| \\
&\le \DeltaHalf \max_{x \in [m,m+\DeltaHalf]}\left|
\frac{d(\KL(x,x+\DeltaHalf))}{dx} 
\right| .
\end{align}
We have
\begin{align}
    \frac{d}{dx} \KL(x,x+\DeltaHalf) & = \frac{d}{dx}
\left(
x\log\left(\frac{x}{x+\DeltaHalf}\right)
+
(1-x)\log\left(\frac{1-x}{1-x-\DeltaHalf}\right)
\right)
\\
& = \left( \log\left(\frac{x}{x+\DeltaHalf}\right) - \log\left(\frac{1-x}{1-x-\DeltaHalf}\right) \right)
+
\left( \frac{\DeltaHalf}{x+ \DeltaHalf} + \frac{\DeltaHalf}{1 - x - \DeltaHalf}\right)
\\
& = \log \left(1 - \frac{\DeltaHalf}{(x+ \DeltaHalf)(1 - x)}\right) + \frac{\DeltaHalf}{(x+ \DeltaHalf)(1 - x - \DeltaHalf)}.
\end{align}
By using 
$
|\frac{\DeltaHalf}{(x+\DeltaHalf)(1-x)}|
< 1/2
$
and $-y-y^2 \le \log(1-y) \le -y$ for $y\in [0,1/2]$, we have
\begin{align*}
  \left| \frac{d}{dx} \KL(x,x+\DeltaHalf) \right|
  & \le \left| - \frac{\DeltaHalf}{(x+ \DeltaHalf)(1 - x)} + \frac{\DeltaHalf}{(x+ \DeltaHalf)(1 - x - \DeltaHalf)}\right| 
  \\
  & \qquad \qquad +  \left| - \left(\frac{ \DeltaHalf}{(x+ \DeltaHalf)(1 - x)}\right)^2\right|
  \\
  &= \frac{\DeltaHalf^2}{(x+ \DeltaHalf)(1-x)(1-x- \Delta)} +   \left(\frac{ \DeltaHalf}{(x+ \DeltaHalf)(1 - x)}\right)^2
\end{align*}
which, by using $\DeltaHalf  < h(T)/2 =  o(T^{-\marg})$ and $x, 1-x > T^{-\marg}$, can be easily bounded by
$
C {\DeltaHalf^2} T^{\margdbl}
$
for some $C>0$,
which is Eq.~\eqref{ineq_bothslide}.

We next derive Eq.~\eqref{ineq_symmetry}.
\begin{align}
\KL(m,u) - \KL(m,l)
&=
\int_{x=0}^{\DeltaHalf} \frac{x}{(m+x)(1-m-x)} dx
- \int_{x=0}^{\DeltaHalf} \frac{x}{(m-x)(1-m+x)} dx\\
&=
\int_{x=0}^{\DeltaHalf} \frac{2x^2(2m-1)}{(m^2 - x^2)((1-x)^2-m^2)} dx
\end{align}
and thus
\begin{align}
|\KL(m,u) - \KL(m,l)|
&\le \DeltaHalf \times \max \left| \frac{2\DeltaHalf^2(2m-1)}{(m^2 - x^2)((1-x)^2-m^2)} \right|\\
&= C \DeltaHalf^3 T^{\marg}
\end{align}
for some $C>0$, which is Eq.~\eqref{ineq_symmetry}.
\endproof

\begin{lem}{\rm (Bound on the KL divergence)}
\label{lem_klfinal}
Let $\hTdecay = \sqrt{C_h\log T/T}$ for some constant $C_h>0$.
Then, there exist $\Capprox >0, T_0$ for all $T\ge T_0$,
\begin{equation}
\frac{(u-l)^2}{2m(1-m)} - \Capprox \Tmargconfrate 
\le \KL(l, u), \KL(u, l) 
\le \frac{(u-l)^2}{2m(1-m)} + \Capprox \Tmargconfrate
\end{equation}
and
\begin{equation}
\frac{(u-l)^2}{8m(1-m)} - \Capprox \Tmargconfrate
\le 
\KL(m, u), \KL(u, m), \KL(m, l), \KL(l, m) 
\le \frac{(u-l)^2}{8m(1-m)} + \Capprox \Tmargconfrate.
\end{equation}
\end{lem}
The proof of Lemma \ref{lem_klfinal} is straightforward from Lemmas \ref{lem_klapprox} and \ref{lem_klsense}.

\section{Lemmas for Upper Bound}

\subsection{Proof of Lemma \ref{lem_inbound}}
\proof{Proof of Lemma \ref{lem_inbound}}
At the end of round $qT$, TSE draws each arm for $N_i(qT) = qT/K$ times.
Eq.~\eqref{ineq_inbound} is derived by using the union bound of the Hoeffding inequality (Eq.~\eqref{ineq_hoeffding}) over $K$ arms. 
\endproof

\subsection{Proof of Lemma \ref{lem_threenearalg}}

\proof{Proof of Lemma \ref{lem_threenearalg}}
\begin{align}
\lefteqn{
\Ex_{\bmu \sim H}\left[ \Ex_{\bmu}[\Ind[\ES,|\bestcand|\ge3]\Deltast]
\right] 
}\\
&\le
4 \Conf \sum_{i,j,k} \int 
\Ind\left[
\mu_i \ge \mu_k - 4\Conf, \mu_j \ge \mu_k - 4\Conf
\right] 
dH(\bmu)\\
&=
4 \Conf \sum_{i,j,k} \int 
\Ind\left[
|\mu_i - \mu_k| \le 4 \Conf, |\mu_j - \mu_k| \le 4 \Conf
\right]  
h_{ij}(\mu_i,\mu_j|\bmuex{ij})
d\mu_i d\mu_j d\Hex{ij}(\bmuex{ij}).
\end{align}
By choosing $\eps=1$ in Eq.~\eqref{ineq_uniformcont}, for any $T$ such that $4\Conf(T) \le \delta(\eps) = \delta(1)$, for each $\bmuex{ij}$, we have
\begin{equation}\label{ineq_smallthreearea}
\int \Ind\left[
|\mu_i - \mu_k| \le 4 \Conf, |\mu_j - \mu_k| \le 4 \Conf
\right]  
d\mu_i d\mu_j
\le 
\int \Ind\left[
\sqrt{(\mu_i - \mu_k)^2 + (\mu_j - \mu_k)^2} \le \sqrt{2}\delta(1)
\right]  
d\mu_i d\mu_j.
\end{equation}
By using this, we have
\begin{align}
\lefteqn{
4 \Conf \int 
\Ind\left[
|\mu_i - \mu_k| \le 4 \Conf, |\mu_j - \mu_k| \le 4 \Conf
\right]  
h(\mu_i,\mu_j|\bmuex{ij})
d\mu_i d\mu_j d\Hex{ij}(\bmuex{ij})
}
\\
&\le
4 \Conf \int
\Ind\left[
|\mu_i - \mu_k| \le 4 \Conf, |\mu_j - \mu_k| \le 4 \Conf
\right]  
\left(
h(\mu_k,\mu_k|\bmuex{ij}) + 1
\right)
d\mu_i d\mu_j d\Hex{ij}(\bmuex{ij})\\
&\text{\ \ \ \ (by uniform continuity and  Eq.~\eqref{ineq_smallthreearea})}\\
&=
O\left( (\Conf)^3\right) =o\left(\frac{1}{T}\right).
\end{align}
\endproof %

\subsection{Lemmas on the main term}

Let $m = (\mu_i + \mu_j)/2$, $\Delta = \mu_i - \mu_j > 0$, and $\Tdhalf = \Tdash/2$.
Let $\hatmu_{i, n}$ be the empirical mean of arm $i$ with the first $n$ samples. 

\begin{lem}{\rm (Tight Bayesian bound)}\label{lem_bayesub_exact}
The following inequality holds: 
\begin{equation}\label{ineq_bayesub_exact}
\int_{\mu_j}^{\min(\mu_j+2\Conf, 1-T^{-\marg})}
(\mu_i-\mu_j)\Prob_{\bmu}[\hatmu_{i, \Tdhalf} \le \hatmu_{j, \Tdhalf}, \ES]
d\mu_i
\le
(1+o(1))
\frac{m(1-m)}{\Tdash}.
\end{equation}
\end{lem}

\begin{remark}{\rm (Lemma \ref{lem_bayesub_exact} is tighter than the Chernoff bound)}
In the proof of Lemma \ref{lem_bayesub_exact}, we carefully use the change-of-measure argument to derive a tight bound.
Alternatively, we may use the concentration inequality (Chernoff bound, Lemma \ref{lem_chernoff}) in bounding the regret, which yields 
\begin{equation}
\Prob[\hatmu_i(T) \le \hatmu_j(T)]
\le 
\Prob[\hatmu_i(T) \le m] + \Prob[\hatmu_j(T) \ge m]
\approx 2 e^{- \Tdhalf \KL(\mu_j, m)}
\end{equation}
which, integrated over the prior, is four times larger than Lemma \ref{lem_bayesub_exact}.
\end{remark}

\begin{lem}{\rm (Tight frequentist bound)}\label{lem_frequb_exact}
Let $\mustarex{ij} < \mu_j < \mu_i < \min(\mu_j+2\Conf, 1-T^{-\marg})$.
Then, 
\begin{equation}\label{ineq_frequb_exact}
\Prob_{\bmu}[\hatmu_{i, \Tdhalf} \le \hatmu_{j, \Tdhalf}, \ES]
=
(1+o(1))
\sqrt{\frac{\Tdash}{8 \pi m (1-m)}}
\int_{s_2=0}^\infty
e^{ 
-\frac{\Tdash}{8m(1-m)}
\left(
\Delta^2
+
2s_2\Delta
+ s_2^2
\right)
}ds_2,
\end{equation}
where the $o(1)$ term does not depend on $\mu_i, \mu_j$. 
\end{lem}

Particular care is required in Lemma \ref{lem_frequb_exact} because a high-probability bound on the KL divergence is not tight (c.f., Section \ref{subsec_tightchallenge});
we can upper-bound a simple regret based on a high-probability bound on the KL divergence, simplifying the analysis of Lemma \ref{lem_frequb_exact}.
However, such a high-probability bound compromises the leading constant.\footnote{Remember that Eq.~\eqref{ineq_twicebound} is twice as large as our bound.}
In the following proof, we use Proposition \ref{prop_stirling}, Lemma~\ref{lem_klfinal}, and manual calculation on the number of combinations because these operations are tight concerning the leading constant.

\proof{Proof of Lemma \ref{lem_frequb_exact}}

Let  
\[
\blambda := (\mu_1, \mu_2, \dots, \mu_{i-1},\underbrace{m}_{\text{$i$-th element}},\mu_{i+1},\dots, \mu_{j-1},\underbrace{m}_{\text{$j$-th element}},\mu_{j+1}, \dots, \mu_K)
\]
be the set of parameters 
where $(\mu_i, \mu_j)$ is replaced by $(m, m)$. 
Let $\hatmu_{i, n}$ be the empirical mean of arm $i$ with $n$ samples.

Under $\mu_j + 2 \Conf > \mu_i > \mu_j > \mustarex{ij}$, we have
\begin{align}
\lefteqn{
\{\hatmu_{i, \Tdhalf} \le \hatmu_{j, \Tdhalf}, \ES\}
}\\
&= 
\{\Tdhalf\,\hatmu_{i, \Tdhalf} \le \Tdhalf\,\hatmu_{j, \Tdhalf}, \ES\}\\
&=
(\Tdhalf\,\hatmu_{i, \Tdhalf}, \Tdhalf\,\hatmu_{j, \Tdhalf}) 
\subseteq
\{(T_1, T_2) \in \Natural^2:
-4T\Conf \le T_1-\Tdhalf m \le T_2-\Tdhalf m \le 4T\Conf  
\}\\
&\ \ \ := \ES_{i,j}
\end{align}

Let $(D_1, D_2) = (T_1 - \Tdhalf m , T_2 - \Tdhalf m)$. 
We have
\begin{align}
\lefteqn{
\Prob_{\bmu}[\hatmu_{i, \Tdhalf } \le \hatmu_{j, \Tdhalf }, \ES]
}\\
&= 
\sum_{(T_1, T_2) \in \ES_{i,j}}
\Prob_{\bmu}\left[
\Tdhalf\,\hatmu_{i, \Tdhalf } = T_1,\,
\Tdhalf\,\hatmu_{j, \Tdhalf } = T_2
\right]\\
&= 
\sum_{(T_1, T_2) \in \ES_{i,j}}
e^{ 
(D_1+\Tdhalf m)\log\left(\frac{\mu_i}{m}\right)
+
(\Tdhalf(1-m)-D_1)\log\left(\frac{1-\mu_i}{1-m}\right)
+
(D_2+\Tdhalf m)\log\left(\frac{\mu_j}{m}\right)
+
(\Tdhalf(1-m)-D_2)\log\left(\frac{1-\mu_j}{1-m}\right)
}\\&\ \ \ \ \ \ \ \ \ \ \times
\Prob_{\blambda}\left[
\Tdhalf\,\hatmu_{i, \Tdhalf } = T_1,
\Tdhalf\,\hatmu_{j, \Tdhalf } = T_2
\right]\\
&\text{\ \ \ \ \ \ \ (change of measure)}\\
&=
\sum_{(T_1, T_2) \in \ES_{i,j}}
e^{ 
-\Tdhalf\left(\KL(m, \mu_i)+\KL(m, \mu_j)\right)
+
\left(
D_1\log\left(\frac{\mu_i}{m}\right)
-
D_1\log\left(\frac{1-\mu_i}{1-m}\right)
+
D_2\log\left(\frac{\mu_j}{m}\right)
-
D_2\log\left(\frac{1-\mu_j}{1-m}\right)
\right)
}\\&\ \ \ \ \ \ \ \ \ \ \times
\Prob_{\blambda}\left[
\Tdhalf\,\hatmu_{i, \Tdhalf } = T_1,\,
\Tdhalf\,\hatmu_{j, \Tdhalf } = T_2
\right]\\
&=
\sum_{(T_1, T_2) \in \ES_{i,j}}
e^{ 
-(1-o(1))\left(\frac{\Tdash\Delta^2}{8m(1-m)}\right)
+
\left(
D_1\log\left(\frac{\mu_i}{m}\right)
-
D_1\log\left(\frac{1-\mu_i}{1-m}\right)
+
D_2\log\left(\frac{\mu_j}{m}\right)
-
D_2\log\left(\frac{1-\mu_j}{1-m}\right)
\right)
}\\&\ \ \ \ \ \ \ \ \ \ \times
\Prob_{\blambda}\left[
\Tdhalf\,\hatmu_{i, \Tdhalf } = T_1,\,
\Tdhalf\,\hatmu_{j, \Tdhalf } = T_2
\right]\\
&\text{\ \ \ \ \ \ \ (By Lemma \ref{lem_klfinal})}\\
&=
\sum_{(T_1, T_2) \in \ES_{i,j}}
e^{ 
-(1-o(1))\left(\frac{\Tdash\Delta^2}{8m(1-m)}
+
D_1\frac{\Delta}{2m(1-m)}
-
D_2\frac{\Delta}{2m(1-m)}
\right)
}\\&\ \ \ \ \ \ \ \ \ \ \times
\Prob_{\blambda}\left[
\Tdhalf\,\hatmu_{i, \Tdhalf } = T_1,\,
\Tdhalf\,\hatmu_{j, \Tdhalf } = T_2
\right]\\
&\text{\ \ \ \ \ \ \ (By $|\log(1+x) - x| = O(x^2)$ and $\left|\frac{\mu_i(1-m)}{m(1-\mu_i)} - 1 - \frac{\Delta}{2m(1-m)}\right| = o(\Delta)$)}\\
&=
\sum_{(T_1, T_2) \in \ES_{i,j}}
e^{ 
-(1-o(1))\left(\frac{\Tdash\Delta^2}{8m(1-m)}
+
D_1\frac{\Delta}{2m(1-m)}
-
D_2\frac{\Delta}{2m(1-m)}
\right)
}\\&\ \ \ \ \ \ \ \ \ \ \times
\frac{\Tdhalf!}{T_1!(\Tdhalf-T_1)!}
m^{T_1}(1-m)^{\Tdhalf-T_1}
\frac{\Tdhalf!}{T_2!(\Tdhalf-T_2)!}
m^{T_2}(1-m)^{\Tdhalf-T_2}\\
&\text{\ \ \ \ \ \ \ \ \ \ (by number of combinations)}\\
&=
\sum_{(T_1, T_2) \in \ES_{i,j}}
e^{ 
-(1-o(1))\left(\frac{\Tdash\Delta^2}{8m(1-m)}
+
D_1\frac{\Delta}{2m(1-m)}
-
D_2\frac{\Delta}{2m(1-m)}
\right)
}\\&\ \ \ \ \ \ \ \ \ \ \times
\frac{1}{\sqrt{2\pi}}
e^{ 
(1+o(1))\left(
(\Tdhalf+\frac{1}{2})\log(\Tdhalf)
-(T_1+\frac{1}{2})\log(T_1)
-(\Tdhalf-T_1+\frac{1}{2})\log(\Tdhalf-T_1)
+
T_1\log(m)
+
(\Tdhalf-T_1)\log(1-m)
\right)
}\\&\ \ \ \ \ \ \ \ \ \ \times
\frac{1}{\sqrt{2\pi}}
e^{ 
(1+o(1))\left(
(\Tdhalf+\frac{1}{2})\log(\Tdhalf)
-(T_2+\frac{1}{2})\log(T_2)
-(\Tdhalf-T_2+\frac{1}{2})\log(\Tdhalf-T_2)
+
T_2\log(m)
+
(\Tdhalf-T_2)\log(1-m)
\right)
}.\\
\label{ineq_newfreq_temp}\\
&\text{\ \ \ \ \ \ \ \ \ \ (by Proposition \ref{prop_stirling})}
\end{align}
Here, letting $(t_1, t_2) = (T_1/\Tdhalf, T_2/\Tdhalf$), we have 
\begin{align}
\lefteqn{
e^{ 
(\Tdhalf+\frac{1}{2})\log(\Tdhalf)
-(T_1+\frac{1}{2})\log(T_1)
-(\Tdhalf-T_1+\frac{1}{2})\log(\Tdhalf-T_1)
+
T_1\log(m)
+
(\Tdhalf-T_1)\log(1-m)
}
}\\
&=
\sqrt{\frac{\Tdhalf}{T_1 (\Tdhalf - T_1)}}
e^{ 
(1+o(1))\left(
\Tdhalf\log(\Tdhalf)
-T_1\log(T_1)
-(\Tdhalf-T_1)\log(\Tdhalf-T_1)
+
T_1\log(m)
+
(\Tdhalf-T_1)\log(1-m)
\right)
}\\
&=
\sqrt{\frac{2}{m (1-m)\Tdash}}
e^{ 
(1+o(1))\left(
\Tdhalf\log(\Tdhalf)
-T_1\log(T_1)
-(\Tdhalf-T_1)\log(\Tdhalf-T_1)
+
T_1\log(m)
+
(\Tdhalf-T_1)\log(1-m)
\right)
}\\
&\text{\ \ \ \ \ \ \ \ \ \ (by $T_1, T_2 = m\Tdhalf + O(T\Conf)$))}\\ 
&=
\sqrt{\frac{2}{m (1-m)\Tdash}}
e^{ 
(1+o(1))\Tdhalf\left(
-t_1\log(t_1)
-(1-t_1)\log(1-t_1)
+
t_1\log(m)
+
(1-t_1)\log(1-m)
\right)
}\\
&= 
\sqrt{\frac{2}{m (1-m)\Tdash}}
e^{ 
-(1-o(1))\Tdhalf\KL(m, t_1)
}\label{ineq_stirling_end}
\end{align}
and thus, by letting $(d_1, d_2) = (D_1/\Tdhalf, D_2/\Tdhalf)$, we have
\begin{align}
\lefteqn{
\text{Eq.~\eqref{ineq_newfreq_temp}}
}\\
&=
\frac{1}{2\pi} \frac{2}{m (1-m)\Tdash}
\sum_{(T_1, T_2) \in \ES_{i,j}}
e^{ 
-(1-o(1))\left(\frac{\Tdash\Delta^2}{8m(1-m)}
+
D_1\frac{\Delta}{2m(1-m)}
-
D_2\frac{\Delta}{2m(1-m)}
+ \Tdhalf\KL(m, t_1)
+ \Tdhalf\KL(m, t_2)
\right)
}\\
&\text{\ \ \ \ \ \ \ \ \ \ (by Eq.~\eqref{ineq_stirling_end})}\\
&=
\frac{1}{\pi m (1-m)\Tdash}
\sum_{(T_1, T_2) \in \ES_{i,j}}
e^{ 
-(1-o(1))\left(\frac{\Tdash\Delta^2}{8m(1-m)}
+
D_1\frac{\Delta}{2m(1-m)}
-
D_2\frac{\Delta}{2m(1-m)}
+ \Tdhalf\KL(m, t_1)
+ \Tdhalf\KL(m, t_2)
\right)
}\\
&=
\frac{1}{\pi m (1-m)\Tdash}
\sum_{(T_1, T_2) \in \ES_{i,j}}
e^{ 
-(1-o(1))\frac{\Tdash}{8m(1-m)}\left(\Delta^2
+
2d_1\Delta
-
2d_2\Delta
+ 2d_1^2 
+ 2d_2^2
\right)
}\\
&\text{\ \ \ \ \ \ \ (By Lemma \ref{lem_klfinal})}\\
&=
\frac{1+o(1)}{\pi m (1-m)\Tdash}
\sum_{(T_1, T_2) \in \ES_{i,j}}
e^{ 
-\frac{\Tdash}{8m(1-m)}\left(\Delta^2
+
2d_1\Delta
-
2d_2\Delta
+ 2d_1^2 
+ 2d_2^2
\right)
}\\
&=
\frac{(1+o(1))\Tdash}{4\pi m (1-m)}
\int_{t_1=-\infty}^{\infty}\int_{t_2=t_1}^\infty
e^{ 
-\frac{\Tdash}{8m(1-m)}\left(
\Delta^2
+
2d_1\Delta
-
2d_2\Delta
+ 2d_1^2 
+ 2d_2^2
\right)
}d(d_1) d(d_2)\label{ineq_ll_normalform}\\ 
&\text{\ \ \ \ \ \ \ \ \ \ (by $dD_1 = \frac{\Tdash}{2} d(d_1), dD_2 = \frac{\Tdash}{2} d(d_2)$)}
\\
&=
\frac{(1+o(1))\Tdash}{8\pi m (1-m)}
\int_{s_1=-\infty}^{\infty}\int_{s_2=0}^\infty
e^{ 
-\frac{\Tdash}{8m(1-m)}\left(
\Delta^2
+
2s_2\Delta
+ s_1^2 
+ s_2^2
\right)
}ds_1 ds_2\\
&\text{\ \ \ \ \ \ \ \ \ \ (by letting $s_1, s_2 = (d_1+d_2), (d_2-d_1)$ and $ds_1 ds_2 = 2 d(d_1) d(d_2)$)}\\
&=
\sqrt{\frac{8m(1-m)}{\pi\Tdash}}
\frac{(1+o(1))\Tdash}{8 m (1-m)}
\int_{s_2=0}^\infty
e^{ 
-\frac{\Tdash}{8m(1-m)}\left(
\Delta^2
+
2s_2\Delta
+ s_2^2
\right)
}ds_2.\\
&\text{\ \ \ \ \ \ \ \ \ \ (by $\int_{-\infty}^\infty \exp(-x^2) dx = \sqrt{\pi}$)}
\end{align}
This concludes the proof. Note that the $o(1)$ term is derived by applying Lemma \ref{lem_klfinal}, where the term $\Capprox \Tmargconfrate$ does not depend on $\mu_i, \mu_j$.
\endproof

\proof{Proof of Lemma \ref{lem_bayesub_exact}}
An integration of Lemma \ref{lem_frequb_exact} over $\Delta = \mu_i - \mu_j$ from $0$ to a sufficiently large value is
\begin{align}
\lefteqn{
(1+o(1))\sqrt{\frac{\Tdash}{8 \pi m (1-m)}}
\int_{\mu_j}^{2\Conf}
\Delta
\int_{s_2=0}^\infty
e^{ 
-\frac{\Tdash}{8m(1-m)}(\Delta+s_2)^2
}ds_2d\Delta
}\\
&\le
(1+o(1))\sqrt{\frac{\Tdash}{8 \pi m (1-m)}}
\int_{\Delta=0}^\infty
\Delta
\int_{s_2=0}^\infty
e^{ 
-\frac{\Tdash}{8m(1-m)}(\Delta+s_2)^2
}ds_2d\Delta\\
&=
(1+o(1))\sqrt{\frac{\Tdash}{8 \pi m (1-m)}}
\int_{s_2=0}^\infty
\int_{\Delta=0}^\infty
\Delta
e^{ 
-\frac{\Tdash}{8m(1-m)}(\Delta+s_2)^2
}d\Delta ds_2
\\
&=
\frac{1+o(1)}{2}\sqrt{\frac{\Tdash}{8 \pi m (1-m)}}
\int_{s_2=0}^\infty
\int_{\Delta=0}^\infty
(\Delta+s_2)
e^{ 
-\frac{\Tdash}{8m(1-m)}
(\Delta+s_2)^2
} d\Delta ds_2 \\
&\text{\ \ \ \ \ \ \ \ \ \ (by $\int_0^\infty \int_0^\infty x e^{-(x+y)^2}dxdy 
= \frac{1}{2}\int_0^\infty \int_0^\infty x e^{-(x+y)^2}dxdy 
+ \frac{1}{2}\int_0^\infty \int_0^\infty y e^{-(x+y)^2}dxdy$)}\\
&=
\frac{1+o(1)}{2}\sqrt{\frac{\Tdash}{8 \pi m (1-m)}}
\int_{s_2=0}^\infty
\left[-\frac{4m(1-m)}{\Tdash}
e^{ 
-\frac{\Tdash}{8m(1-m)}
(\Delta+s_2)^2
}
\right]_{\Delta=0}^\infty
ds_2 \\
&=
(1+o(1))\sqrt{\frac{m (1-m)}{2 \pi \Tdash}}
\int_{s_2=0}^\infty
e^{ 
-\frac{\Tdash}{8m(1-m)}
s_2^2
}ds_2 \\
&=
(1+o(1))
\frac{m(1-m)}{\Tdash}.\\
&\text{\ \ \ \ \ \ \ \ \ \ (by $\int_{0}^\infty \exp(-x^2) dx = \sqrt{\pi}/2$)}
\end{align}

\endproof %

\subsection{Proof of Lemma \ref{lem_besttwo}}

\proof{Proof of Lemma \ref{lem_besttwo}} 
Event $\ES$ implies that $\ist \in \bestcand$. 
Under $\bestcand=\{i,j\}$, we have $N_i(T) = N_j(T) = \Tdhalf$. 
\begin{align}
\lefteqn{
\Ex_{\bmu \sim H}\left[ \Ex_{\bmu}[\Ind[\ES,|\bestcand|=2] \Deltast] 
\right] 
}\\
&\le 
\sum_{i,j}
\Ex_{\bmu \sim H}\left[ \Ind[\ES,\mu_i > \mu_j > \mustarex{ij}]\Ex_{\bmu}[\Ind[\ES,\bestcand=\{i,j\}] \Deltast] \right]\\
&\le
\sum_{i,j}
\Ex_{\bmu \sim H}\left[ \Ind[\mu_j + 2 \Conf > \mu_i > \mu_j > \mustarex{ij}, ]\Ex_{\bmu}[\Ind[\ES,\bestcand=\{i,j\}]\Deltast] \right]\\
&=
\sum_{i,j}
\Ex_{\bmu \sim H}\left[ \Ind[\mu_j + 2 \Conf > \mu_i > \mu_j > \mustarex{ij}](\mu_i - \mu_j) \Prob_{\bmu}[\hatmu_{i, \Tdhalf} \le \hatmu_{j, \Tdhalf}, \ES] \right]\\
&= 
\sum_{i,j}
\Ex_{\bmu \sim H}\left[ \Ind[\mu_j + 2 \Conf > \mu_i > \mu_j > \mustarex{ij}, 1-T^{-\marg}>\mu_i](\mu_i - \mu_j) \Prob_{\bmu}[\hatmu_{i, \Tdhalf} \le \hatmu_{j, \Tdhalf}, \ES] \right] \\
&\ \ \ \ \ + \sum_{i,j} 2\Conf O\left(
\Prob_{\bmu \sim H}[\mu_i> 1-T^{-\marg}, |\mu_i-\mu_j|\le 2 \Conf]
\right). \\
&= 
\sum_{i,j}
\Ex_{\bmu \sim H}\left[ \Ind[\mu_j + 2 \Conf > \mu_i > \mu_j > \mustarex{ij}, 1-T^{-\marg}>\mu_i](\mu_i - \mu_j) \Prob_{\bmu}[\hatmu_{i, \Tdhalf} \le \hatmu_{j, \Tdhalf}, \ES] \right] + o\left(\frac{1}{T}\right) \\
&\text{\ \ \ \ \ (by uniform continuity)}\\
&\le 
\sum_i
\Ex_{\bmu \sim H}\left[ \Ind[\mustarex{i} + 2 \Conf >\mu_i>  \mustarex{i}, 1-T^{-\marg}>\mu_i](\mu_i - \mustarex{i}) \Prob_{\bmu}[\hatmu_{i, \Tdhalf} \le \hatmu_{j, \Tdhalf}, \ES] 
\right] + o\left(\frac{1}{T}\right) \label{ineq_besttwo_split}.
\end{align}
Here, the first term of Eq.~\eqref{ineq_besttwo_split} is bounded as:
\begin{align}
\lefteqn{
\Ex_{\bmu \sim H}\left[ \Ind[\mustarex{i} + 2 \Conf  >\mu_i>  \mustarex{i}, 1-T^{-\marg}>\mu_i](\mu_i - \mustarex{i}) \Prob_{\bmu}[\hatmu_{i, \Tdhalf} \le \hatmu_{j, \Tdhalf}, \ES] 
\right]
}\\
&\le
\int_{[0,1]^{K-1}}
\int_{\mu_i=\mustarex{i}}^{\min(\mustarex{i} + 2 \Conf, 1-T^{-\marg})}
(\mu_i - \mustarex{i}) 
\Prob_{\bmu}[\hatmu_{i, \Tdhalf} \le \hatmu_{j, \Tdhalf}, \ES]
h_i(\mu_i|\bmuex{i}) 
d\mu_i d\Hex{i}(\bmuex{i}) \\
&\le
(1 + o(1)) 
\int_{[0,1]^{K-1}}
h_i(\mustarex{i}|\bmuex{i}) 
\int_{\mu_i=\mustarex{i}}^{\min(\mustarex{i} + 2 \Conf, 1-T^{-\marg})}
(\mu_i - \mustarex{i}) 
\Prob_{\bmu}[\hatmu_{i, \Tdhalf} \le \hatmu_{j, \Tdhalf}, \ES]
d\mu_i d\Hex{i}(\bmuex{i}) \\
&\text{\ \ \ (by uniform continuity)}\\
&\le 
(1 + o(1))
\int_{[0,1]^{K-1}}
h_i(\mustarex{i}|\bmuex{i}) 
\frac{\mustarex{i}(1-\mustarex{i})}{\Tdash}
d\Hex{i}(\bmuex{i}),\\
&\text{\ \ \ (by the union bound of Lemma \ref{lem_bayesub_exact} over all $j\ne i$)}
\end{align}
which completes the proof.
\endproof %

\section{Lemmas for Lower Bound}
\subsection{Lower bound on the error}

\begin{prop}{\rm (Lemma 1 of \citet{kaufman16a})}\label{prop_nbound}
Let $\bmu, \bnu \in [0,1]^K$ be two sets of model parameters. Then, for any event $\EE$, the following inequality holds:
\begin{equation} 
\sum_{i \in [K]} \Ex_{\bnu}[N_i(T)] \KL(\nu_i, \mu_i) 
\ge 
\KL(\Prob_{\bnu}(\EE), \Prob_{\bmu}(\EE)).
\end{equation}
Moreover, 
\begin{equation}\label{ineq_xox}
\forall_{x \in [0,1]} \KL(x, 1-x) \ge \log\frac{1}{2.4x}.
\end{equation}
\end{prop}
Proposition \ref{prop_nbound} describes the hardness of identifying two different sets of parameters.
By using this proposition, we derive Lemma \ref{lem_prooflb_kcg}.
\proof{Proof of Lemma \ref{lem_prooflb_kcg}}
We assume that Eq.~\eqref{ineq_prooflb_kcg} is false and derive a contradiction; that is, suppose that
\begin{equation*}
\max(\Prob_{\bmu}[J(T) \ne i], \Prob_{\bnu}[J(T) \ne j]) 
< 
\frac{1}{2.4} \exp\Big( -(1+\eta)T\KL(\mu_j, \mu_i) \Big).
\end{equation*}
Let $p = \frac{1}{2.4} \exp\left( -(1+\eta)T\KL(\mu_j, \mu_i) \right)$. Then, Proposition~\ref{prop_nbound} with $\EE = \{J(T) \ne i\}$ yields 
\begin{align} 
T \max\left(\KL(\mu_j, \mu_i), \KL(\mu_i, \mu_j)\right)
&\ge 
\sum_{i \in [K]} \Ex_{\bnu}[N_i(T)] \KL(\nu_i, \mu_i) \\
&\ge \KL(\Prob_{\bnu}[J(T) \ne i], \Prob_{\bmu}[J(T) \ne i])\\
&\text{\ \ \ \ (by Proposition~\ref{prop_nbound})}\\
&\ge \KL(p, 1-p)\\
&\text{\ \ \ \ (by Eq.~\eqref{ineq_prooflb_kcg} is false)}\\
&\ge \log\frac{1}{2.4p}\\
&\text{\ \ \ \ (by Eq.~\eqref{ineq_xox})}\\
&\ge (1+\eta)T\KL(\mu_j, \mu_i) \\
&\ge (1+\eta)(1-\eta') T\max\left(\KL(\mu_j, \mu_i), \KL(\mu_i, \mu_j)\right),\\
&\text{\ \ \ \ (for some constant $C>0$ and $\eta' = C \times \Conf T^{\margdbl}$ by Lemma \ref{lem_klsense})}
\end{align}
which contradicts for some $C_2 > 0$ and $\eta = C_2 \times \Conf T^{\margdbl}$ such that $(1+\eta)(1-\eta') > 1$, and thus Eq.~\eqref{ineq_prooflb_kcg} holds.
\endproof

\subsection{Exchangeable mass}

\proof{Proof of Lemma \ref{lem_exchange}}
Letting $m = (\mu_i+\mu_j)/2$, we have
\begin{align}
\lefteqn{
\int_{\Theta_{i,j,\marg}} 
(\mu_i - \mu_j)
 f(\mu_i,\mu_j,\bmuex{ij}) dH(\bmu)
}\\
&=
\int_{\Theta_{i,j,\marg}} 
(\mu_i - \mu_j)
 f(\mu_i,\mu_j,\bmuex{ij}) 
h_{ij}(\mu_i, \mu_j|\bmuex{ij}) d\mu_i d\mu_j d\Hex{ij}(\bmuex{ij})\\
&= 
(1-o(1))
\int_{\Theta_{i,j,\marg}} 
h_{ij}(m, m|\bmuex{ij})
(\mu_i - \mu_j)
 f(\mu_i,\mu_j,\bmuex{ij}) 
d\mu_i d\mu_j d\Hex{ij}(\bmuex{ij})\\
&\text{\ \ \ \ \ (by uniform continuity and the diameter of $\Theta_{i,j,\marg}$ is $o(1)$)}\\
&= 
(1-o(1)) 
\int_{\Theta_{j,i,\marg}} 
h_{ij}(m, m|\bmuex{ij})
(\mu_j - \mu_i)
 f(\mu_j,\mu_i,\bmuex{ij}) 
d\mu_i d\mu_j d\Hex{ij}(\bmuex{ij})\\
&\text{\ \ \ \ \ (by symmetry)}\\
&= 
(1-o(1)) 
\int_{\Theta_{j,i,\marg}} 
(\mu_j - \mu_i)
 f(\mu_j,\mu_i,\bmuex{ij}) dH(\bmu).\\
&\text{\ \ \ \ \ (by uniform continuity)}\\
\end{align}
\endproof

\subsection{Integration on the lower bound}

\proof{Proof of Lemma \ref{lem_integration}}
\begin{align}
\lefteqn{
\int_{[\mustarex{i}, \mustarex{i}+2\Conf]} (\mu_i - \mustarex{i})
\exp\left( -(1+\eta)T\KL(\mustarex{i}, \mu_i) \right)
d\mu_i
}\\
&=
(1-o(1))\int_{[\mustarex{i}, \mustarex{i}+2\Conf]} (\mu_i - \mustarex{i})
\exp\left( -(1+\eta)T \frac{(\mu_i - \mustarex{i})^2}{2\mustarex{i}(1-\mustarex{i})} \right)
d\mu_i\\
&\text{\ \ \ \ \ (by Lemma \ref{lem_klfinal})}\\
&=
(1-o(1))
\left[
\frac{\mustarex{i}(1-\mustarex{i})}{T}
\exp\left( -(1+\eta)T \frac{(\mu_i - \mustarex{i})^2}{2\mustarex{i}(1-\mustarex{i})} \right)
\right]_{\mustarex{i}}^{\mustarex{i}+2\Conf}\\
&=
\frac{\mustarex{i}(1-\mustarex{i})}{(1+\eta)T} - o\left(\frac{1}{T}\right).\\
&\text{\ \ \ \ \ (by $\exp\left( -(1+\eta)T\frac{(\mu_i - \mustarex{i})^2}{2\mustarex{i}(1-\mustarex{i})} \right) = o(1)$ for $\mu_i - \mustarex{i} = 2\Conf$)}
\end{align}
\endproof %

\end{document}

%% file: preamble.tex
\usepackage[utf8]{inputenc}
\usepackage[driver=dvipdfm,margin=1in]{geometry}
\usepackage{mathtools}
\usepackage{amsmath}
\usepackage{amsthm}
\usepackage{amssymb}
\usepackage{setspace}
\usepackage{booktabs}
\usepackage{tabularx}

\mathtoolsset{showonlyrefs}  

\makeatletter

\usepackage{amsthm}
\usepackage{amsfonts}
\usepackage{bbm}
\usepackage{graphics}
\usepackage{enumerate}
\usepackage{float}
\usepackage{caption}
\usepackage{subcaption}

\theoremstyle{remark}

\theoremstyle{definition}
\newtheorem{thm}{Theorem}
\newtheorem{lem}[thm]{Lemma}

\newtheorem{prop}[thm]{Proposition}

\newtheorem{remark}{Remark}
\newtheorem{example}{Example}

\newtheorem{assp}{Assumption}

\usepackage{amsmath} 
\usepackage{amssymb}
\usepackage{bm}
\usepackage{ascmac}
\usepackage{setspace}
\usepackage[at]{easylist}

\usepackage[hyphens]{url}
\usepackage[colorlinks,urlcolor=blue, citecolor=blue, menucolor=blue]{hyperref}
\theoremstyle{plain}

\usepackage{algorithm,algpseudocode}

\usepackage{amsmath}
\DeclareMathOperator*{\argmax}{arg\,max}

\newcommand{\Natural}{\mathbb{N}}
\newcommand{\Regret}{\mathrm{Reg}}

\newcommand{\SRegBayes}{\mathrm{R}_{H}}

\newcommand{\SRegFreq}{\mathrm{R}_{\bmu}}
\newcommand{\marg}{1/6} 
\newcommand{\margdbl}{1/3} 
\newcommand{\Tmargconfrate}{T^{-1/6}(\log T)^3} 
\newcommand{\Deltast}{\Delta_{J(T)}} 

\newcommand{\Ex}{\mathbb{E}}
\newcommand{\Prob}{\mathbb{P}}
\newcommand{\eps}{\epsilon}

\newcommand{\Ind}{\bm{1}}

\newcommand{\nn}{\nonumber\\}

\newcommand{\ist}{i^*}

\newcommand{\jst}{{\setminus i}^*}
\newcommand{\mustarex}[1]{\mu_{\setminus #1}^*}

\newcommand{\EA}{\mathcal{A}}
\newcommand{\EB}{\mathcal{B}}

\newcommand{\EE}{\mathcal{E}}

\newcommand{\EJ}{\mathcal{J}}

\newcommand{\ES}{\mathcal{S}}

\newcommand{\bmu}{\bm{\mu}}

\newcommand{\bestcand}{\hat{\EJ}^*}

\newcommand{\hatmu}{\hat{\mu}}

\newcommand{\KL}{d_{\mathrm{KL}}}
\newcommand{\Unif}{\mathrm{Unif}}

\newcommand{\Conf}{B_{\mathrm{conf}}}
\newcommand{\hTdecay}{h(T)}
\newcommand{\Bernoulli}{\mathrm{Bernoulli}}

\newcommand{\Hex}[1]{H_{\setminus #1}}

\newcommand{\bmuex}[1]{\bm{\mu}_{\setminus #1}}

\newcommand{\blambda}{\bm{\lambda}}
\newcommand{\bnu}{\bm{\nu}}

\newcommand{\Cbayesopt}{C_{\mathrm{opt}}}
\newcommand{\CKL}{C_{\mathrm{KL}}}

\newcommand{\Capprox}{C_a}
\newcommand{\Runnum}{10^5}
\newcommand{\Hcond}[1]{H_{#1}}
\newcommand{\hcond}[1]{h_{#1}}
\newcommand{\Tdash}{T'}
\newcommand{\Tdhalf}{T'_h}
\newcommand{\DeltaHalf}{\Delta_h}